\documentclass[letterpaper,10pt,conference]{ieeeconf}
\IEEEoverridecommandlockouts
\usepackage{amsmath,amssymb}
\usepackage{bm}

\usepackage{amsthm}
\usepackage{cite}
\usepackage{graphicx}
\usepackage{url}
\usepackage{tikz}
\usepackage{tikz-qtree}
\usetikzlibrary{positioning,automata,backgrounds}
\usepackage{xcolor}
\usepackage{algorithm}
\usepackage[noend]{algpseudocode} 
\usepackage{subcaption} 
\usepackage{mathtools}
\usepackage{hyperref}
\usepackage{booktabs}
\usepackage{multirow}
\usepackage{bbm}
\usepackage{siunitx}
\usepackage{empheq} 
\usepackage{makecell}  

\theoremstyle{plain}
\newtheorem{theorem}{Theorem}
\newtheorem{lemma}{Lemma}

\newtheorem{definition}{Definition}
\newtheorem{assumption}{Assumption}

\newtheorem{problem}{Problem}
\newtheorem{remark}{Remark}

\newcommand{\xb}{\bm{x}}
\newcommand{\yb}{\bm{y}}
\newcommand{\zb}{\bm{z}}
\newcommand{\bb}{\bm{b}}
\newcommand{\ub}{\bm{u}}

\newcommand{\Ab}{\bm{A}}
\newcommand{\Bb}{\bm{B}}
\newcommand{\mub}{\bm{\mu}}
\newcommand{\Lambdab}{\bm{\Lambda}}
\newcommand{\Sigmab}{\bm{\Sigma}}
\newcommand{\Wb}{\bm{W}}
\newcommand{\Hb}{\bm{H}}
\newcommand{\Vb}{\bm{V}}
\newcommand{\Nb}{\bm{N}}
\newcommand{\Ib}{\bm{I}}

\newcommand{\omegab}{\bm{\omega}}
\newcommand{\nub}{\bm{\nu}}

\newcommand{\VaR}{\mathrm{VaR}}

\DeclareMathOperator{\erf}{erf}

\usepackage{xspace}
\newcommand{\envRA}{\texttt{reach-avoid}\xspace}
\newcommand{\envCM}{\texttt{cautious-motion}\xspace}

\newcommand{\until}{{\bf{U}}}
\newcommand{\always}{\square}
\newcommand{\eventually}{\lozenge}

\begin{document}

\title{Belief Space Planning for Risk Signal Temporal Logic \\ via Mixed-Integer Convex Programming}

\author{Shaohang Han, Joris Verhagen, Kaier Liang, and Jana Tumova
\thanks{This work was partially supported by the Wallenberg AI, Autonomous Systems and Software Program (WASP) funded by the Knut and Alice Wallenberg Foundation.}%
\thanks{The authors are with the Division of Robotics, Perception and Learning, School of Electrical Engineering and Computer Science, KTH Royal Institute of Technology, Stockholm, Sweden, and are also affiliated with Digital Futures. Email: \texttt{shaohang@kth.se}.}%
}

\maketitle
\thispagestyle{empty}

\begin{abstract}
Signal Temporal Logic (STL) specifies robot tasks with spatio-temporal requirements, and its robustness semantics allow planning by trajectory optimization. In practice, robots estimate their state from noisy measurements, and when the information gained from a measurement depends on the state, the robot can actively reduce its uncertainty by planning in the belief space. The resulting problem is, however, generally non-convex, for which gradient-based methods are neither complete nor globally optimal. In this paper, we consider Risk STL (RiSTL) planning in Gaussian belief space and show that an information-form parameterization of the belief, together with a convex relaxation of the belief dynamics, yields a mixed-integer convex program (MICP). We prove that the MICP is sound and complete, and that the relaxation preserves the globally optimal input trajectories. We validate the effectiveness of our method in numerical simulations.
\end{abstract}

\section{Introduction}
Complex robot tasks with spatio-temporal requirements can be specified in Signal Temporal Logic (STL)~\cite{maler2004monitoring}. For example, we may require \textit{an underwater robot to reach each inspection target within $60$ seconds and hover there for $10$ seconds, to come to the surface for a GPS fix between minutes three and four, and to slow down whenever it is close to a target}. Beyond specifying these requirements, STL also admits quantitative semantics through the notion of \textit{robustness}, which measures the degree to which a trajectory satisfies a specification~\cite{belta2019formal}. This allows us to formulate a trajectory optimization problem that maximizes robustness against spatial perturbations, which is a common approach to STL planning.

A common assumption in STL planning relies on perfect knowledge of the state, but real-world robots estimate their state from noisy and incomplete sensor measurements. An effective approach is Bayesian filtering, which maintains a \textit{belief}: a probability distribution over possible states~\cite{thrun2006probabilistic}. Several extensions of STL quantify this uncertainty and synthesize motion plans with satisfaction guarantees~\cite{sadigh2016safe, safaoui2020control,kordabad2024distributionally,yu2026signal}. In these methods, the uncertainty evolves independently of the plan. In many settings, however, the information gained from a measurement depends on the state, and the robot can \textit{actively reduce} its uncertainty by moving to where its sensors are informative. Capturing this coupling requires planning over the belief dynamics, with the uncertainty itself as a decision variable. This leads to belief space temporal logic planning, where the predicates are defined on the belief. 

Existing belief space temporal logic planning finds satisfying trajectories by search over finite abstractions~\cite{vasile2016control,da2019active,ho2023planning}, and does not optimize robustness. In contrast, \cite{vahs2023risk} follows the trajectory optimization formulation of STL and lifts the risk predicates of RiSTL~\cite{safaoui2020control} to the Gaussian belief space, where the belief is represented by its mean and covariance. It maximizes robustness directly in the continuous belief space. The resulting problem is, however, a non-convex nonlinear program (NLP) solved with gradient-based methods, which optimize a smooth reformulation of the robustness~\cite{gilpin2020smooth,han2026exact}. Such solvers are sound but neither complete nor globally optimal. Depending on the initialization, they may converge to a poor local optimum, or fail to find a feasible solution even when one exists.

In state space, for linear systems with linear predicates, mixed-integer encodings avoid this limitation. They represent the logical and temporal operators using binary variables~\cite{belta2019formal,kurtz2022mixed}, and result in a mixed-integer convex program (MICP) that is sound, complete, and solvable to global optimality. In belief space, however, the problem remains non-convex even when dynamics, measurements, and predicates are all linear. There are two reasons. First, the covariance evolves through a nonlinear update. Second, the risk predicates depend nonlinearly on the covariance. Our key finding is that both challenges can be addressed by parameterizing the belief in information form~\cite{thrun2006probabilistic}, which replaces the covariance by its inverse, the precision matrix. This parameterization allows us to turn the risk predicates into convex constraints. When the measurement information is a piecewise concave function of the mean, we show that the belief dynamics admit a convex relaxation that preserves the optimal input trajectories. The planning problem then becomes an MICP, which can be solved with off-the-shelf solvers.

Our contributions are threefold. First, we derive an MICP for RiSTL planning in Gaussian belief space, based on the information-form parameterization and a convex relaxation of the belief dynamics. Second, we prove that our method is sound, complete, and globally optimal. Third, our simulation experiments demonstrate that the planner achieves meaningful spatio-temporal tasks over polygonal regions, and that information gathering emerges from the specification, as the robot detours toward informative regions to satisfy RiSTL predicates without an explicit exploration objective.




\section{Preliminaries}
\label{sec:preliminary_problem_formulation}
\subsection{Gaussian Belief Dynamics}
\label{sec:belief_dynamics}
We consider a discrete-time linear process model with Gaussian process noise and a linear measurement model with state-dependent Gaussian noise:
\begin{subequations}\label{eq:system}
\begin{align}
    \xb_{t+1} &= \Ab_t\xb_t+\Bb_t\ub_t+\omegab_t, \\
    \yb_{t+1} &= \Hb_{t+1}\xb_{t+1}+\nub_{t+1},
\end{align}
\end{subequations}
where $\xb_t\in\mathbb R^{n_x}$ is the state, $\ub_t\in\mathbb R^{n_u}$ is the control input, and $\omegab_t\sim\mathcal N(0,\Wb_t)$ and $\nub_{t+1}\mid\xb_{t+1}\sim\mathcal N(0,\Vb_{t+1}(\xb_{t+1}))$ are process and measurement noise, with the covariance matrices $\Wb_t\succ0$ and $\Vb_{t+1}(\xb)\succ0$. 

Given a Gaussian initial belief, we can estimate the posterior by $\mathcal N(\mub_t,\Sigmab_t)$ using the Gaussian filtering update, where $\mub_t$ and $\Sigmab_t$ denote the mean and covariance. Unlike the covariance parameterization in belief space planning~\cite{vahs2023risk}, we use the \textit{information} form~\cite{thrun2006probabilistic,bostrom2018global}, which describes the belief by its mean and precision $\Lambdab_t=\Sigmab_t^{-1}$. The belief state is therefore defined as $\bb_t=[\mub_t^{\top},\operatorname{vec}(\Lambdab_t)^{\top}]^{\top}$, where $\operatorname{vec}$ stacks matrix columns. The initial belief $\bb_0$ has mean $\mub_{\mathrm{init}}$ and precision $\Lambdab_{\mathrm{init}}\succ0$.

We adapt the Kalman and information filter updates in~\cite{thrun2006probabilistic} to the mean and precision. The prediction step is
\begin{align}
    \bar{\mub}_{t+1}
    &:=\Ab_t\mub_t+\Bb_t\ub_t, \label{eq:predicted_mean}\\
    \bar{\Lambdab}_{t+1}
    &:=(\Ab_t\Lambdab_t^{-1}\Ab_t^{\top}+\Wb_t)^{-1}. \label{eq:prediction_step}
\end{align}
Here $\bar{\mub}_{t+1}$ and $\bar{\Lambdab}_{t+1}$ are the predicted mean and precision before receiving $\yb_{t+1}$. For the measurement update, we evaluate $\Vb_{t+1}$ at the predicted mean $\bar{\mub}_{t+1}$. The updated precision, Kalman gain $\bm K_{t+1}$, and updated mean are
\begin{align}
    \Lambdab_{t+1}
    &=\bar{\Lambdab}_{t+1}+\mathcal I_{t+1}(\bar{\mub}_{t+1}), \label{eq:precision_update}\\
    \bm K_{t+1}
    &=\Lambdab_{t+1}^{-1}\Hb_{t+1}^{\top}\Vb_{t+1}(\bar{\mub}_{t+1})^{-1}, \label{eq:kalman_gain}\\
    \mub_{t+1}
    &=\bar{\mub}_{t+1}+\bm K_{t+1}(\yb_{t+1}-\Hb_{t+1}\bar{\mub}_{t+1}). \label{eq:mean_update}
\end{align}
Here, we define 
\begin{equation}\label{eq:fisher_information}
    \mathcal I_{t+1}(\bar{\mub}_{t+1})
    :=\Hb_{t+1}^{\top}\Vb_{t+1}(\bar{\mub}_{t+1})^{-1}\Hb_{t+1}
\end{equation}
to be the measurement information function. Since $\Vb_{t+1}\succ0$ is continuous, $\mathcal I_{t+1}$ is continuous and $\mathcal I_{t+1}\succeq0$. We also have that $\Lambdab_t\succ0$ for all $t=0,\ldots,T$, which can be seen by induction from $\Lambdab_{\mathrm{init}}\succ0$, since $\Wb_t\succ0$ implies $\bar{\Lambdab}_{t+1}\succ0$ in~\eqref{eq:prediction_step} and $\mathcal I_{t+1}\succeq0$ then implies $\Lambdab_{t+1}\succ0$ in~\eqref{eq:precision_update}. 

Similar to \cite{vahs2023risk,Platt-RSS-10}, we assume the following maximum likelihood observation (MLO) condition.

\begin{assumption}
\label{as:mlo}
Under the Gaussian approximation, future measurements equal their predicted values, i.e., $\yb_{t+1}=\Hb_{t+1}\bar{\mub}_{t+1}$.
\end{assumption}

Under Assumption~\ref{as:mlo}, the measurement equals its predicted value, so the innovation term in~\eqref{eq:mean_update} is zero. We therefore obtain $\mub_{t+1}=\bar{\mub}_{t+1}$ and the following belief dynamics.
\begin{equation}\label{eq:belief_dynamics}
\begin{aligned}
    \bb_{t+1}:=\begin{bmatrix}\mub_{t+1}\\\operatorname{vec}(\Lambdab_{t+1})\end{bmatrix}=\begin{bmatrix}
    \Ab_t\mub_t+\Bb_t\ub_t\\
    \operatorname{vec}\!\left(\bar{\Lambdab}_{t+1}+\mathcal I_{t+1}\right)
    \end{bmatrix}.
\end{aligned}
\end{equation}
In the following, we denote the belief, mean, and input trajectories by $\bb=(\bb_0,\ldots,\bb_T)$, $\mub=(\mub_0,\ldots,\mub_T)$, and $\ub=(\ub_0,\ldots,\ub_{T-1})$, respectively.

\subsection{Belief Space Risk Signal Temporal Logic}
Risk Signal Temporal Logic (RiSTL) extends STL with predicates that account for uncertainty through risk measures~\cite{safaoui2020control}. In Gaussian belief space, these predicates become deterministic functions of the belief states, and are then combined using the same logical and temporal operators as in STL~\cite{vahs2023risk}. In this work, we consider formulas in a positive normal form (PNF), with more details in~\cite{belta2019formal}. 

\begin{definition}[Time-bounded RiSTL]\label{def:syntax}
We consider a closed bounded time interval $I=[t_1,t_2]\subset\mathbb Z_{\ge0}$, where $t_1\le t_2$. The RiSTL formulas in PNF are recursively defined as
\begin{equation}\label{eq:syntax}
    \varphi::=p\mid\varphi_1\land\varphi_2\mid\varphi_1\lor\varphi_2\mid\always_I\varphi\mid\eventually_I\varphi\mid\varphi_1\until_I\varphi_2,
\end{equation}
where $p$ is an atomic predicate. The logical operators $\land$ and $\lor$ denote ``and'' and ``or''. The temporal operators $\always_I$ and $\eventually_I$ require that $\varphi$ should hold $\forall t\in I$ or $\exists t\in I$ respectively. The temporal operator $\until_I$ requires $\varphi_1$ to hold until $\varphi_2$ holds within $I$. 
\end{definition}

We consider the \textit{spatial robustness}, which quantifies the degree of satisfaction of a RiSTL specification. We denote the robustness of $(\bb,\ub)$ at time $t$ by $\rho^\varphi(\bb,\ub,t)$. 

\begin{definition}[Spatial robustness]\label{def:robustness}
Each atomic predicate $p$ is associated with a real-valued function $h^p$ of the belief states and control inputs. The robustness semantics are recursively defined as follows:
\begingroup
\setlength{\jot}{3pt}
\setlength{\abovedisplayskip}{2pt}
\setlength{\belowdisplayskip}{2pt}
\setlength{\abovedisplayshortskip}{0pt}
\setlength{\belowdisplayshortskip}{0pt}
\begin{align*}
& (\bb,\ub)\models\varphi \Leftrightarrow \rho^\varphi(\bb,\ub,0)\ge0 \\
& \rho^p(\bb,\ub,t)=h^p(\bb_t,\ub_t) \\
& \rho^{\phi_1\land\phi_2}(\bb,\ub,t)
    =\min\!\big([\rho^{\phi_1}(\bb,\ub,t),\,\rho^{\phi_2}(\bb,\ub,t)]^\top\big) \\
& \rho^{\phi_1\lor\phi_2}(\bb,\ub,t)
    =\max\!\big([\rho^{\phi_1}(\bb,\ub,t),\,\rho^{\phi_2}(\bb,\ub,t)]^\top\big) \\
& \rho^{\always_I\phi_1}(\bb,\ub,t)
    =\min\!\big([\rho^{\phi_1}(\bb,\ub,t')]_{t'\in t+I}^\top\big) \\
& \rho^{\eventually_I\phi_1}(\bb,\ub,t)
    =\max\!\big([\rho^{\phi_1}(\bb,\ub,t')]_{t'\in t+I}^\top\big) \\
& \rho^{\phi_1\until_I\phi_2}(\bb,\ub,t)
    =\max\!\Big(\Big[\min\!\Big(\big[\rho^{\phi_2}(\bb,\ub,t'),\\
&\hspace{14mm}\min\!\big([\rho^{\phi_1}(\bb,\ub,t'')]_{t''\in[t,t')}^\top\big)\big]^\top\Big)\Big]_{t'\in t+I}^\top\Big).
\end{align*}
\endgroup
Intervals contain integer times and $\min\emptyset=+\infty$. The until operator requires its left operand to hold strictly before the witness time $t'$, but not at $t'$ itself.
\end{definition}

\section{Problem Formulation}
\label{sec:problem_formulation}
We first define the atomic predicates used in the specification and then formulate the trajectory optimization problem.

\subsection{Atomic Predicates}
We consider the following three types of atomic predicates and use a robustness margin $\eta\in\mathbb R$ in their predicate constraints $h^p\ge\eta$.

\textit{Risk predicates.} We construct risk predicates over Gaussian beliefs from linear state space predicates. For a vector $\Nb\ne0$ and a scalar $D\in\mathbb R$, the state space predicate is
\begin{equation}\label{eq:linear_state_predicate}
    \bar h_{\mathrm{lin}}(\xb)=\Nb^{\top}\xb-D\ge0.
\end{equation}
Under the Gaussian belief $\xb\sim\mathcal N(\mub,\Lambdab^{-1})$, we require $\Pr[\bar h_{\mathrm{lin}}(\xb)\ge0]\ge1-\tau$ for a user-specified risk level $\tau\in(0,1)$.
We define the corresponding risk predicate using a tail risk measure called Value-at-Risk, similar to~\cite{vahs2023belief}.
\begin{definition}[Value-at-Risk]\label{def:var}
For $\xb\sim\mathcal N(\mub,\Lambdab^{-1})$, we let $Y:=\bar h_{\mathrm{lin}}(\xb)$. We define Value-at-Risk as its lower $\tau$-quantile $\VaR_{\tau}(Y):=\inf\{\ell\in\mathbb R:\Pr[Y\le\ell]\ge\tau\}$.
\end{definition}
The risk predicate function is therefore $h^{p}_{\mathrm{lin}}(\mub,\Lambdab):=\VaR_{\tau}(Y)$. Since $Y$ is Gaussian with mean $\Nb^{\top}\mub-D$ and variance $\Nb^{\top}\Lambdab^{-1}\Nb$, we obtain
\begin{equation}\label{eq:gaussian_var}
    h^{p_{\mathrm{lin}}}(\mub,\Lambdab)=\Nb^{\top}\mub-D-\kappa_\tau\sqrt{2\Nb^{\top}\Lambdab^{-1}\Nb}.
\end{equation}
Here $\kappa_\tau:=\erf^{-1}(1-2\tau)$, and $\erf$ denotes the error function. The risk predicate constraint $h^{p}_{\mathrm{lin}}\ge\eta$ is equivalent to the chance constraint $\Pr[\Nb^{\top}\xb\ge D+\eta]\ge1-\tau$.

\begin{remark}\label{rem:negated_predicate}
We can also construct risk predicates from negated linear state space predicates, by replacing $(\Nb,D)$ with $(-\Nb,-D)$ in the state space predicate~\eqref{eq:linear_state_predicate} and then using the risk predicate function~\eqref{eq:gaussian_var}.
\end{remark}

\textit{Uncertainty predicates.} We can directly constrain uncertainty by defining a predicate on the belief state. For a given bound $\varepsilon\ge0$, we can define
\begin{equation}\label{eq:uncertainty_margin}
    h^{p_{\mathrm{unc}}}(\Lambdab)
    =\varepsilon-\operatorname{tr}(\Lambdab^{-1})\ge0,
\end{equation}
which bounds the total state variance by $\varepsilon$. Since $\operatorname{tr}(\Lambdab^{-1})$ is convex on $\Lambdab\succ0$, the predicate function $h^{p_{\mathrm{unc}}}$ is concave, and its predicate constraint $h^{p_{\mathrm{unc}}}\ge\eta$ is convex in $(\Lambdab,\eta)$.

\textit{Control-magnitude predicates.} For a given bound $u_{\max}\ge0$, these predicates constrain the control input:
\begin{equation}\label{eq:slow_margin}
    h^{p_{\mathrm{slow}}}(\ub)=u_{\max}-\|\ub\|_2\ge0.
\end{equation}
This predicate bounds velocity when $\ub$ is a velocity command. Since $\|\ub\|_2$ is convex, the predicate function $h^{p_{\mathrm{slow}}}$ is concave, therefore $h^{p_{\mathrm{slow}}}\ge\eta$ is convex in $(\ub,\eta)$.

\subsection{Trajectory Optimization}
The specification $\varphi$ is built from the above predicates using~\eqref{eq:syntax}. We seek control inputs and a corresponding belief trajectory that satisfy $\varphi$, while maximizing spatial robustness and minimizing the control effort $\mathcal J_{\mathrm{u}}(\ub):=\sum_{t=0}^{T-1}\|\ub_t\|_2^2$.
We use the common robustness margin $\eta$ as a continuous auxiliary decision variable that lower-bounds $\rho^\varphi(\bb,\ub,0)$.
\begin{problem}\label{prob:exact}
Given system~\eqref{eq:system}, an initial belief, and a specification $\varphi$, solve
\begin{subequations}\label{problem_1}
\begin{align}
    \min_{\bb,\ub,\eta}\quad
    &\alpha\mathcal J_{\mathrm{u}}(\ub)-\eta \label{problem_1:obj}\\
    \text{s.t.}\quad
    &\mub_{t+1}=\Ab_t\mub_t+\Bb_t\ub_t, \label{problem_1:dyn}\\
    &\Lambdab_{t+1}=\bar{\Lambdab}_{t+1}+\mathcal I_{t+1}, \label{problem_1:information}\\
    &\mub_0=\mub_{\mathrm{init}},\quad
    \Lambdab_0=\Lambdab_{\mathrm{init}}, \label{problem_1:init}\\
    &\rho^\varphi(\bb,\ub,0)\ge\eta, \label{problem_1:rob}\\
    &\eta\ge0. \label{problem_1:eta}
\end{align}
\end{subequations}
where $\alpha\ge0$. The mean and precision variables use $t=0,\ldots,T$, while dynamics and controls use $t=0,\ldots,T-1$.
\end{problem}

\begin{remark}
Constraint~\eqref{problem_1:eta} enforces satisfaction of $\varphi$. Setting $\alpha=0$ gives pure robustness maximization.
\end{remark}

We can obtain a direct mixed-integer nonlinear program (MINLP) formulation of Problem~\ref{prob:exact} by encoding the logical and temporal operators in~\eqref{problem_1:rob} using binary variables~\cite{belta2019formal}. The precision update~\eqref{problem_1:information} is a non-convex equality constraint, and the risk predicate constraint $h^{p_{\mathrm{lin}}}(\mub,\Lambdab)\ge\eta$ is not convex in general, depending on the risk level $\tau$. To obtain an MICP, we introduce a convex relaxation of the precision update in Section~\ref{sec:belief_reformulation}, and establish convexity of the predicate constraints for the choice of risk level in Section~\ref{sec:atomic_encoding}. Then in Section~\ref{sec:theoretical_analysis}, we show that the resulting MICP preserves the feasible and globally optimal input trajectories of the direct MINLP formulation.

\section{Mixed-Integer Convex Reformulation}
\label{sec:micp}
\subsection{Convex Reformulation of Belief Dynamics}
\label{sec:belief_reformulation}
Inspired by~\cite{bostrom2018global}, we relax the nonlinear precision update~\eqref{problem_1:information} while keeping positive definiteness:
\begin{equation}\label{eq:information_relaxation}
    0\prec\Lambdab_{t+1}\preceq\bar{\Lambdab}_{t+1}+\mathcal I_{t+1}.
\end{equation}
The initial precision is fixed by~\eqref{problem_1:init}. Similar to~\cite{bostrom2018global}, we apply Woodbury's matrix identity to get
{\small\begin{equation}\label{eq:woodbury}
    \bar{\Lambdab}_{t+1}=\Wb_t^{-1}-\Wb_t^{-1}\Ab_t\!(\Lambdab_t+\Ab_t^{\top}\Wb_t^{-1}\Ab_t)^{-1}\!\Ab_t^{\top}\Wb_t^{-1}.
\end{equation}}%
Since $\Wb_t\succ0$ and $\Lambdab_t\succ0$, we have $\Lambdab_t+\Ab_t^{\top}\Wb_t^{-1}\Ab_t\succ0$. Therefore, we can apply the Schur complement \cite{vandenberghe2004convex} to upper bound~\eqref{eq:information_relaxation} as
{\small\begin{equation}\label{eq:information_block}
    \mathcal P_t:=\begin{bmatrix}
    \Wb_t^{-1}+\mathcal I_{t+1}-\Lambdab_{t+1}
    &\Wb_t^{-1}\Ab_t\\
    \Ab_t^{\top}\Wb_t^{-1}
    &\Lambdab_t+\Ab_t^{\top}\Wb_t^{-1}\Ab_t
    \end{bmatrix}\succeq0.
\end{equation}}%
We have thus replaced~\eqref{problem_1:information} by imposing $\Lambdab_t\succ0$ and the block constraint~\eqref{eq:information_block}. To obtain a convex feasible set, we first consider $\mathcal I_t(\mub_t)$ that are matrix-concave in $\mub_t$.

\begin{lemma}\label{lem:information_convexity}
For matrix-concave measurement information functions $\mathcal I_t$, the relaxed information constraints~\eqref{eq:information_relaxation} define a convex feasible set.
\end{lemma}
\begin{proof}
For the upper bound, $\mathcal P_t$ is the sum of an affine function of $\Lambdab_t,\Lambdab_{t+1}$ and a block matrix containing $\mathcal I_{t+1}(\mub_{t+1})$. Since $\mathcal I_{t+1}$ is matrix-concave, $\mathcal P_t$ is jointly matrix-concave in $(\Lambdab_t,\Lambdab_{t+1},\mub_{t+1})$. Consequently, for each $t$, the constraint $\mathcal P_t\succeq0$ defines a convex set in $(\Lambdab_t,\Lambdab_{t+1},\mub_{t+1})$, and $\Lambdab_t\succ0$ defines a convex set in $\Lambdab_t$. Each of these sets is also convex when viewed as a set of the full trajectory variables $(\mub_0,\ldots,\mub_T,\Lambdab_0,\ldots,\Lambdab_T)$, since the remaining variables are unconstrained. The feasible set of~\eqref{eq:information_relaxation} is the intersection of these convex sets and is therefore convex.
\end{proof}

Interestingly, the relaxation~\eqref{eq:information_relaxation} enlarges the feasible set compared to~\eqref{problem_1:information}, but does not change the optimal input trajectories. We provide a formal analysis in Section~\ref{sec:theoretical_analysis}.

Our approach also applies when $\mathcal I_t$ is not matrix-concave, provided it is the maximum of finitely many matrix-concave pieces. A representative example is shown in Fig.~\ref{fig:reach_avoid}, where measurements are accurate near a few separate locations. Specifically, we consider matrix-concave pieces $\mathcal I_t^{(0)},\ldots,\mathcal I_t^{(K)}$ and define
\begin{equation}\label{eq:piecewise_information_branches}
    \mathcal I_t(\mub)=\max_{\preceq}\bigl\{\mathcal I_t^{(k)}(\mub):k=0,\ldots,K\bigr\}.
\end{equation}
Here $\max_{\preceq}$ denotes the maximum in the positive semidefinite order, which requires that at every $\mub$ some $k^*\in\{0,\ldots,K\}$ satisfies $\mathcal I_t^{(k^*)}(\mub)\succeq\mathcal I_t^{(j)}(\mub)$ for all $j\in\{0,\ldots,K\}$. For each $k$, replacing $\mathcal I_{t+1}$ in~\eqref{eq:information_block} by $\mathcal I_{t+1}^{(k)}$ gives a convex constraint $\mathcal P_t^{(k)}\succeq0$. Since $\mathcal I_{t+1}\succeq\mathcal I_{t+1}^{(k)}$, we have $\mathcal P_t\succeq\mathcal P_t^{(k)}$ for every $k$, with equality for a piece attaining the maximum. Hence, $\mathcal P_t\succeq0$ holds if and only if $\mathcal P_t^{(k)}\succeq0$ for at least one $k$. We note that the maximum of matrix-concave functions is in general not matrix-concave. We therefore use binary variables to select one of these convex constraints, which makes the problem convex once the binary variables are fixed.

\subsection{Convexity of Risk Predicate Constraints}
\label{sec:atomic_encoding}
To make the risk predicate constraint $h^{p_{\mathrm{lin}}}\ge\eta$ with~\eqref{eq:gaussian_var} convex, we choose the risk level $\tau\in(0,1/2)$, which requires a satisfaction probability greater than $1/2$. This is justified by the fact that we favor robust satisfaction, and hence a small risk level $\tau$. This choice ensures $\kappa_\tau>0$, which makes the risk predicate constraint convex, as we show below.

\begin{lemma}\label{lem:linear_risk_convexity}
For fixed $\Nb\ne0$, $D\in\mathbb R$, and $\tau\in(0,1/2)$, the constraint $h^{p_{\mathrm{lin}}}(\mub,\Lambdab)\ge\eta$ defines a convex set in $(\mub,\Lambdab,\eta)$ on $\Lambdab\succ0$.
\end{lemma}
\begin{proof}
To study the convexity of $f(\Lambdab):=\sqrt{\Nb^{\top}\Lambdab^{-1}\Nb}$, we look at its epigraph $f(\Lambdab)\le s$. Since $\Nb\ne0$, we have $f(\Lambdab)>0$ and hence $s>0$. Applying the Schur complement first with respect to $\Lambdab$ and then with respect to $s^2$ gives
\begin{equation}
\begin{aligned}
    f(\Lambdab)\le s
    &\iff\begin{bmatrix}
    \Lambdab&\Nb\\
    \Nb^{\top}&s^2
    \end{bmatrix}\succeq0\\
    &\iff\Lambdab-s^{-2}\Nb\Nb^{\top}\succeq0.
\end{aligned}
\end{equation}
Since $-s^{-2}$ is concave on $s>0$ and $\Nb\Nb^{\top}\succeq0$, the final matrix expression is jointly matrix-concave in $(\Lambdab,s)$. Thus the epigraph is convex, and $f$ is convex. Since $\kappa_\tau>0$, $h^{p_{\mathrm{lin}}}(\mub,\Lambdab)$ is concave, and the constraint $h^{p_{\mathrm{lin}}}(\mub,\Lambdab)\ge\eta$ is convex in $(\mub,\Lambdab,\eta)$.
\end{proof}

\begin{remark}\label{rem:covariance_predicate}
If we parameterize the belief state by its mean and covariance, as in~\cite{vahs2023risk}, the same risk predicate constraint takes the form $\Nb^{\top}\mub-D-\kappa_\tau\sqrt{2\Nb^{\top}\Sigmab\Nb}\ge\eta$, which is generally non-convex in $(\Sigmab,\eta)$ when $\tau\in(0,1/2)$. By using the precision $\Lambdab=\Sigmab^{-1}$, we obtain an equivalent convex constraint by Lemma~\ref{lem:linear_risk_convexity}, while preserving the predicate values.
\end{remark}

\subsection{Final MICP Formulation}
\label{sec:micp_formulation}
We denote the precision trajectory by $\Lambdab=(\Lambdab_0,\ldots,\Lambdab_T)$ and collect all binary decision variables introduced below into a vector $\zb$. By Definition~\ref{def:robustness}, $\rho^\varphi(\bb,\ub,0)$ is a nested $\min$/$\max$ expression over the predicate values $h^{p}(\bb_t,\ub_t)$. This structure permits the standard mixed-integer encoding of~\cite{belta2019formal}, which uses binary variables to select the predicate constraints $h^{p}(\bb_t,\ub_t)\ge\eta$ whose satisfaction ensures $\rho^\varphi(\bb,\ub,0)\ge\eta$. We use the set $\mathcal C_\varphi$ to group the resulting constraints on $(\mub,\Lambdab,\ub,\eta,\zb)$. For a given trajectory, $\mathcal C_\varphi$ is feasible for some choice of $\zb$ if and only if $\rho^\varphi(\bb,\ub,0)\ge\eta$. Since the predicate constraints are convex by Section~\ref{sec:atomic_encoding}, $\mathcal C_\varphi$ is mixed-integer convex. 

We use the same time indices as in~\eqref{problem_1}. For the piecewise information model~\eqref{eq:piecewise_information_branches}, we introduce additional binary variables in $\zb$ to select one constraint $\mathcal P_t^{(k)}\succeq0$ at each time step in place of~\eqref{problem_micp:information}. Finally, we obtain the following MICP.
\begin{subequations}\label{problem_micp}
\begin{align}
    \min_{\mub,\ub,\Lambdab,\eta,\zb}\quad
    &\alpha\mathcal J_{\mathrm{u}}(\ub)-\eta\\
    \text{s.t.}\quad
    &\mub_0=\mub_{\mathrm{init}},\quad
    \Lambdab_0=\Lambdab_{\mathrm{init}},\\
    &\mub_{t+1}=\Ab_t\mub_t+\Bb_t\ub_t,\\
    &\Lambdab_t\succ0,\\
    &\mathcal P_t(\Lambdab_t,\Lambdab_{t+1},\mub_{t+1})\succeq0,\label{problem_micp:information}\\
    &(\mub,\Lambdab,\ub,\eta,\zb)\in\mathcal C_\varphi,\\
    &\eta\ge0.
\end{align}
\end{subequations}

\section{Theoretical Analysis}
\label{sec:theoretical_analysis}
We now show that the MICP~\eqref{problem_micp} preserves the feasible and globally optimal input trajectories of the direct MINLP formulation of Problem~\ref{prob:exact}. The equivalence concerns the feasible pairs $(\ub,\eta)$ for the same initial belief. 

\begin{lemma}[Equivalence of feasible input trajectories]
\label{lem:feasible_input_trajectories}
For the predicates in Section~\ref{sec:problem_formulation}, if $\mathcal I_t$ is either matrix-concave on $\mathbb R^n$ or piecewise as in~\eqref{eq:piecewise_information_branches}, then the direct MINLP formulation of Problem~\ref{prob:exact} and the MICP~\eqref{problem_micp} have exactly the same feasible pairs $(\ub,\eta)$.
\end{lemma}
\begin{proof}
For the piecewise $\mathcal I_t$, selecting a feasible piece is equivalent to satisfying~\eqref{eq:information_block} with a matrix-concave $\mathcal I_t^{(k)}$. Hence, we let $\mathcal I_t$ be matrix-concave for the following analysis without loss of generality.

\textit{(MINLP to MICP)} We consider an arbitrary feasible solution of the MINLP formulation of Problem~\ref{prob:exact}, with continuous variables $(\mub,\ub,\Lambdab,\eta)$. Since $\Lambdab_t\succ0$ and the precision update~\eqref{problem_1:information} holds, the relaxed constraint~\eqref{eq:information_relaxation} also holds. The robustness constraint~\eqref{problem_1:rob} implies that binary variables $\zb$ can be chosen to satisfy $\mathcal C_\varphi$. All remaining constraints of~\eqref{problem_micp} are also imposed in~\eqref{problem_1}. Therefore, these values of $(\mub,\ub,\Lambdab,\eta)$, together with $\zb$, are feasible for the MICP, proving feasibility of the same pair $(\ub,\eta)$.

\textit{(MICP to MINLP)} We consider an arbitrary feasible solution of the MICP with input trajectory $\ub$, precision trajectory $\Lambdab^{\mathrm{re}}$, and margin $\eta$. We propagate the exact belief dynamics~\eqref{eq:belief_dynamics} from the same initial belief using $\ub$, and denote the resulting precision trajectory by $\Lambdab^{\mathrm{ex}}$. We denote the corresponding belief trajectories by $\bb^{\mathrm{re}}$ and $\bb^{\mathrm{ex}}$. They share the same mean because the initial belief and inputs are the same. We first show that $\Lambdab_t^{\mathrm{ex}}\succeq\Lambdab_t^{\mathrm{re}}\succ0$ for every $t$. 
The relaxed trajectory satisfies $\Lambdab_t^{\mathrm{re}}\succ0$ by~\eqref{eq:information_relaxation}. We proceed by induction on $t$. At $t=0$, we have $\Lambdab_0^{\mathrm{ex}}=\Lambdab_0^{\mathrm{re}}=\Lambdab_{\mathrm{init}}\succ0$. Suppose $\Lambdab_t^{\mathrm{ex}}\succeq\Lambdab_t^{\mathrm{re}}\succ0$ holds for an arbitrary $t\in\{0,\ldots,T-1\}$. Then~\eqref{eq:prediction_step} gives $\bar{\Lambdab}_{t+1}^{\mathrm{ex}}\succeq\bar{\Lambdab}_{t+1}^{\mathrm{re}}$. The means coincide, so both trajectories have the same $\mathcal I_{t+1}$. Combining the exact update with~\eqref{eq:information_relaxation} yields $\Lambdab_{t+1}^{\mathrm{ex}}\succeq\Lambdab_{t+1}^{\mathrm{re}}\succ0$. This completes the induction step, so the claim holds for every $t$.

Consequently, $(\Lambdab_t^{\mathrm{ex}})^{-1}\preceq(\Lambdab_t^{\mathrm{re}})^{-1}$. For the predicates that we consider in Section~\ref{sec:problem_formulation}, the variance and trace terms in~\eqref{eq:gaussian_var} and~\eqref{eq:uncertainty_margin} satisfy $\Nb^{\top}(\Lambdab_t^{\mathrm{ex}})^{-1}\Nb\le\Nb^{\top}(\Lambdab_t^{\mathrm{re}})^{-1}\Nb$ and $\operatorname{tr}((\Lambdab_t^{\mathrm{ex}})^{-1})\le\operatorname{tr}((\Lambdab_t^{\mathrm{re}})^{-1})$, respectively. Since $\kappa_\tau>0$ and the means are the same, the first inequality implies that the risk predicate function~\eqref{eq:gaussian_var} does not decrease. The second inequality gives the same conclusion for the uncertainty predicate~\eqref{eq:uncertainty_margin}. The control-magnitude predicate~\eqref{eq:slow_margin} has the same value because the inputs are the same. Hence each of these predicates satisfies $h^p(\bb_t^{\mathrm{ex}},\ub_t)\ge h^p(\bb_t^{\mathrm{re}},\ub_t)$. The $\min$ and $\max$ operations in the robustness semantics are nondecreasing in each argument. Using the MICP constraints $\mathcal C_\varphi$, we therefore obtain $\rho^\varphi(\bb^{\mathrm{ex}},\ub,0)\ge\rho^\varphi(\bb^{\mathrm{re}},\ub,0)\ge\eta\ge0$. The input trajectory therefore satisfies the specification under the original belief dynamics and is feasible for the MINLP with the same $\eta$. Therefore, the two formulations have exactly the same feasible pairs $(\ub,\eta)$.
\end{proof}

Given the equivalence of feasible input trajectories, we can show that the MICP preserves global optimality.
\begin{theorem}[Global optimality]
\label{thm:global_optimality}
Under the conditions of Lemma~\ref{lem:feasible_input_trajectories}, if feasible, the MINLP formulation of Problem~\ref{prob:exact} and the MICP~\eqref{problem_micp} attain the same optimal value and have the same globally optimal input trajectories.
\end{theorem}
\begin{proof}
By Lemma~\ref{lem:feasible_input_trajectories}, both formulations have the same feasible pairs $(\ub,\eta)$. Since both minimize $\alpha\mathcal J_{\mathrm{u}}(\ub)-\eta$, they have the same optimal value and globally optimal input trajectories.
\end{proof}

We can also express the feasibility result in terms of soundness and completeness.
\begin{theorem}[Soundness and completeness]
\label{thm:soundness_completeness}
Under the conditions of Lemma~\ref{lem:feasible_input_trajectories}, every feasible input trajectory of the MICP~\eqref{problem_micp} is feasible for Problem~\ref{prob:exact} (soundness). Conversely, every feasible input trajectory of Problem~\ref{prob:exact} is feasible for the MICP (completeness).
\end{theorem}
\begin{proof}
Both statements follow from the two directions proved in Lemma~\ref{lem:feasible_input_trajectories}.
\end{proof}

\section{Experiments}
\label{sec:experiments}

\subsection{Experimental Setup}
\label{sec:experimental_setup}
We consider a planar robot with single-integrator dynamics, where the control input $\ub_t$ is a velocity, in a light-dark domain~\cite{Platt-RSS-10}, where the noise of the position measurements depends on the position. Specifically, $\Hb_t=\Ib$ and $\Vb_t(\mub)=\Ib/m(\mub)$, where $m:\mathbb R^n\to\mathbb R_{>0}$ is large in the light and small in the dark, so that the measurement information~\eqref{eq:fisher_information} is $\mathcal I_t(\mub)=m(\mub)\Ib$. We use $m$ that is either affine or the maximum of a constant and concave quadratic functions, so that each piece in~\eqref{eq:piecewise_information_branches} is concave. Regions are convex polygons, each given by a set of linear state space predicates~\eqref{eq:linear_state_predicate}, similar to~\cite{kurtz2022mixed}. We use the same symbol $R$ for a polygon and for the formula that requires the belief to be inside it. This formula is the conjunction of the risk predicates~\eqref{eq:gaussian_var} of its state space predicates. Similarly, $\overline{R}$ requires the belief to be outside $R$, and is the disjunction of the risk predicates constructed from the negated state space predicates, as in Remark~\ref{rem:negated_predicate}. We adapt two examples from \cite{vahs2023risk}, \envRA and \envCM.

\subsubsection{\envRA}
The robot should avoid two obstacles and reach one of two goals within $T=15$ steps:
\begin{equation}\label{eq:spec_reach_avoid}
    \varphi_1=\always_{[0,15]}\overline{O}_1\land\always_{[0,15]}\overline{O}_2\land\bigl(\eventually_{[0,15]}G_1\lor\eventually_{[0,15]}G_2\bigr).
\end{equation}
Two lights at $\bm q_1$ and $\bm q_2$ give $m=\max\{m_{\mathrm{dark}},f_1,f_2\}$, where $m_{\mathrm{dark}}>0$ is a small value in the dark and $f_k(\mub)=m_{\mathrm{light}}-\beta_k\|\mub-\bm q_k\|_2^2$ with $m_{\mathrm{light}}>m_{\mathrm{dark}}$ and $\beta_k>0$, so that $K=2$ in~\eqref{eq:piecewise_information_branches}. The gap between $O_1$ and $O_2$ is narrow compared to the initial uncertainty, so the robot cannot pass through it before reducing its uncertainty. The smaller and farther goal $G_2$ is in the light.

\subsubsection{\envCM}
The robot starts inside $G_3$ with $T=35$. The function $m$ is affine, positive throughout the workspace, and increases with $\mu_y$, so the workspace becomes brighter toward the top. We require the robot to stay inside $G_3$ until its uncertainty is low, to move slowly whenever its uncertainty is not low, and to eventually reach $G_4$:
\begin{equation}\label{eq:spec_cautious}
    \varphi_2=\always_{[0,34]}\bigl(p_{\mathrm{unc}}\lor p_{\mathrm{slow}}\bigr)\land\eventually_{[0,35]}G_4\land G_3\,\until_{[0,35]}\,p_{\mathrm{unc}},
\end{equation}
with $\varepsilon=0.05$ in~\eqref{eq:uncertainty_margin} and $u_{\max}=0.35$ in~\eqref{eq:slow_margin}.

\subsection{Implementation Details}
In general, the block constraint~\eqref{eq:information_block} and the predicate constraints for~\eqref{eq:gaussian_var} and~\eqref{eq:uncertainty_margin} are linear matrix inequalities in $\Lambdab_t$, so~\eqref{problem_micp} is a mixed-integer semidefinite program (MISDP). Since $\Ab_t$, $\Wb_t$, $\Lambdab_{\mathrm{init}}$, and $\mathcal I_t$ are diagonal in our setup, the belief dynamics~\eqref{eq:belief_dynamics} preserve diagonal precision matrices, and we restrict $\Lambdab_t$ in~\eqref{problem_micp} accordingly. The block constraint~\eqref{eq:information_block} and the predicate constraints for~\eqref{eq:gaussian_var}, \eqref{eq:uncertainty_margin}, and~\eqref{eq:slow_margin} then admit second-order cone (SOC) representations~\cite{mosekcookbook}, so~\eqref{problem_micp} becomes a mixed-integer second-order cone program (MISOCP). We encode the SOC constraints as convex quadratic constraints and the selection of predicate constraints in $\mathcal C_\varphi$ with big-$M$ constants. We set $\alpha=0.01$ and solve the MISOCP with Gurobi~\cite{gurobi} on an Intel Core i7-13700H laptop. We then propagate~\eqref{eq:belief_dynamics} with the optimal inputs and report the robustness of the exact belief trajectory. As a baseline, we solve Problem~\ref{prob:exact} as an NLP following~\cite{vahs2023risk}, with the smooth robustness of~\cite{gilpin2020smooth} and IPOPT~\cite{wachter2006implementation} through Drake~\cite{tedrake2019drake}. We use eight restarts initialized with RRT* paths to the goal regions and their perturbations, and report the best local optimum.

\begin{figure}[t]
    \centering
    \begin{subfigure}[t]{0.49\linewidth}
        \centering
        \includegraphics[width=\linewidth]{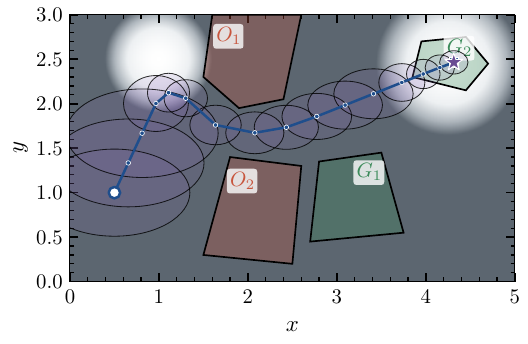}
        \caption{\envRA.}
        \label{fig:reach_avoid}
    \end{subfigure}
    \hfill
    \begin{subfigure}[t]{0.49\linewidth}
        \centering
        \includegraphics[width=\linewidth]{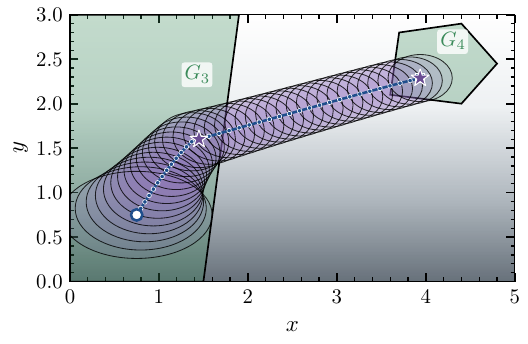}
        \caption{\envCM.}
        \label{fig:cautious_motion}
    \end{subfigure}
    \caption{Optimal belief trajectories. The background shows $m$, where brighter means more informative measurements. The blue line is the mean trajectory with $95\%$ confidence ellipses. Stars mark the time step at which a goal is reached and, in (b), at which $p_{\mathrm{unc}}$ first holds.}
    \label{fig:trajectories}
\end{figure}

\subsection{Results}
\label{sec:results}
Fig.~\ref{fig:trajectories} shows the optimal belief trajectories of the MICP.

In \envRA (Fig.~\ref{fig:reach_avoid}), the initial belief is too uncertain to pass between $O_1$ and $O_2$. Therefore, the robot first detours to the light at $\bm q_1$, and then passes between the obstacles. Although $G_1$ is closer, the robot continues to $G_2$, where the light reduces the uncertainty further.

In \envCM (Fig.~\ref{fig:cautious_motion}), the robot moves slowly for $t\le12$ and stays inside $G_3$. Instead of moving toward $G_4$, it moves toward the upper part of $G_3$, where measurements are more informative. At $t=13$, the uncertainty decreases and $p_{\mathrm{unc}}$ holds. The robot then increases its speed and reaches $G_4$. The robustness is small because the margins of $p_{\mathrm{slow}}$, of $p_{\mathrm{unc}}$, and of $G_4$ are equal at the optimum. In both examples, information gathering emerges from the specification without an explicit exploration objective.

\begin{table}[t]
    \caption{Comparison of the MICP and the NLP baseline on the exact belief trajectories. First feasible is the time to the first feasible solution of the MICP, which satisfies the specification.}
    \label{tab:results}
    \centering
    \small
    \setlength{\tabcolsep}{3pt}
    \begin{tabular}{lcccc}
        \toprule
        & \multicolumn{2}{c}{\envRA} & \multicolumn{2}{c}{\envCM} \\
        \cmidrule(lr){2-3}\cmidrule(lr){4-5}
        & MICP & NLP & MICP & NLP \\
        \midrule
        First feasible [\si{s}] & $2.1$ & -- & $1.9$ & -- \\
        Runtime [\si{s}] & $\mathbf{83.0}$ & 170.1 & $\mathbf{5.6}$ & 204.1 \\
        Robustness $\rho^{\varphi}$ & $\mathbf{0.146}$ & 0.129 & $\mathbf{0.0047}$ & $-0.057$ \\
        Cost $\alpha\mathcal J_{\mathrm{u}}-\rho^{\varphi}$ & $\mathbf{-0.065}$ & $-0.041$ & $\mathbf{0.059}$ & $0.120$ \\
        \bottomrule
    \end{tabular}
\end{table}

Table~\ref{tab:results} compares the MICP with the NLP baseline. In \envRA, the NLP converges to a local optimum with lower robustness and higher cost than the MICP. In \envCM, the NLP solver does not find a feasible solution with the current initialization. 

\section{Conclusion}
We showed that RiSTL planning in Gaussian belief space can be formulated as an MICP. Parameterizing the belief in information form makes the risk predicates convex, and relaxing the belief dynamics to a convex constraint preserves the feasible and globally optimal input trajectories. The planner is therefore sound, complete, and globally optimal. We validated it in numerical simulations, where the MICP certifies globally optimal plans and information gathering emerges from the specification.

\bibliographystyle{IEEEtran}
\bibliography{references}
	
\end{document}